\documentclass{article}
\usepackage[a4paper,margin=1in]{geometry}
\usepackage{amsmath,amssymb,mathtools,amsthm}
\usepackage[utf8]{inputenc}
\usepackage{graphicx}
\usepackage{authblk}

\usepackage{microtype}
\usepackage{booktabs}
\usepackage{bbm}

\usepackage{amsmath}
\usepackage{xcolor}
\usepackage{cancel}
\usepackage[hidelinks]{hyperref}
\usepackage{enumitem}
\usepackage{natbib}
\usepackage{outlines}
\usepackage{nicefrac}
\theoremstyle{definition}

\newtheorem{theorem}{Theorem}[section]
\newtheorem{lemma}[theorem]{Lemma}
\newtheorem{remark}[theorem]{Remark}

\newtheorem{corollary}[theorem]{Corollary}

\theoremstyle{definition}

\newtheorem{example}[theorem]{Example}

\theoremstyle{remark}

\usepackage{mathrsfs}

\title{Time dependent loss reweighting for flow matching and diffusion models is theoretically justified}

\author[1]{Lukas Billera\thanks{Correspondence to:
\href{mailto:lukas.billera@ki.se}{lukas.billera@ki.se} and
\href{mailto:benjamin.murrell@ki.se}{benjamin.murrell@ki.se}}}
\author[1]{Hedwig Nora Nordlinder}
\author[1]{Ben Murrell$^*$}

\affil[1]{Department of Microbiology, Tumor and Cell Biology, Karolinska Institutet}

\date{\today}

\begin{document}

\maketitle

\begin{abstract}
This brief note clarifies that, in Generator Matching (which subsumes a large family of flow matching and diffusion models over continuous, manifold, and discrete spaces), both the Bregman divergence loss and the linear parameterization of the generator can depend on both the current state $X_t$ and the time $t$, and we show that the expectation over time in the loss can be taken with respect to a broad class of time distributions. We also show this for Edit Flows, which falls outside of Generator Matching. That the loss can depend on $t$ clarifies that time-dependent loss weighting schemes, often used in practice to stabilize training, are theoretically justified when the specific flow or diffusion scheme is a special case of Generator Matching (or Edit Flows). It also often simplifies the construction of $X_1$-predictor schemes, which are sometimes preferred for model-related reasons. We show examples that rely upon the dependence of linear parameterizations, and of the Bregman divergence loss, on $t$ and $X_t$.

\end{abstract}
\section{Introduction}
Diffusion \citep{ho2020denoisingdiffusionprobabilisticmodels, song2021scorebasedgenerativemodelingstochastic} and Flow Matching \citep{lipmanflowcite} approaches to generative modeling both train a model to transport samples from a simple distribution to the data distribution. This family of approaches was recently subsumed under `Generator Matching', which characterizes a general space of states, processes, and losses under which conditional samples can be used to train a marginal generator.

The loss used for these models generally involves i) a term that compares the model's prediction to some function of the conditional process generator, and ii) a time-dependent scaling constant prescribed by the specific theoretical framework used to justify the method. Since the outset of the field (e.g., \cite{ho2020denoisingdiffusionprobabilisticmodels}) right up until recently \citep{nguyen2025oneflowconcurrentmixedmodalinterleaved} the prescribed time-dependent loss scaling is discarded or adjusted in practice, usually motivated empirically and often without theoretical justification (with exceptions, e.g. \cite{zhang2025trajectoryflowmatchingapplications, Oresten2025.11.03.686219}). 

With unlimited compute, if you trained a separate model for each time $t$ and state $X_t$, scaling each model's loss by a different positive constant should not change what that model learns, suggesting that time and state dependent loss scaling is justified and merely a practical issue. Here we formalize this notion, and extend it in ways that are trickier to heuristically reason about.

\section{Preliminaries}
Throughout this note, we make use of notation and definitions similar to those found in Generator Matching (GM) \citep{gmcite} and Flow Matching Guide (FMG) \citep{Lipman2024FlowMG}. Following FMG, generator matching prescribes a conditional probability path $p_t(dx|z)$ interpolating between a simple initial distribution $p_0$ and a target distribution $p_1$ on a state space $S$, conditioned on a latent state $z\in \mathcal Z$. With a distribution $p_Z$ over $\mathcal Z$, we consider the marginal probability path $p_t(dx)$ defined via the two-stage sampling procedure
\[Z \sim p_Z,  \ (X_t|Z=z) \sim p_{t|Z}(dx|z) \implies X_t \sim p_t(dx). \]
We assume the GM-specified regularity conditions hold (cf. Appendix~\ref{app:gm}), denoting our state space by $S$ and our class of test functions by $\mathcal T$. Under those conditions, the trajectories of a process $X_t$ are determined by i) its initial value, and 2) the infinitesimal generator --- meaning that if a neural network can learn to parametrize the infinitesimal generator, we then can transport from the initial distribution to the target distribution. We aim to parametrize the infinitesimal generator $\mathcal L_t$ of the process $X_t$ by a neural network $\mathcal L_t^\theta$.

 FMG suggests the following form of a linear parametrization of the generator of the process $X_t$: 
\[\mathcal L_tf(x) = \langle \mathcal Kf(x), F_t(x) \rangle_x\]
for a function $F_t$ such that $F_t(x) \in \Omega_x$, where $\Omega_x \subset V_x$ is a closed, convex subset of the inner product space $(V_x, \langle\cdot,\cdot\rangle_x).$ Analogously, FMG defines a linear parameterization of the conditional infinitesimal generator by
\[\mathcal L_t^zf(x) = \langle \mathcal Kf(x), F_t^z(x)\rangle_x\]
for a function $F_t^z$ such that $F_t^z(x) \in \Omega_x$. Then for Bregman divergences $D_{x} : \Omega_x\times \Omega_x\to \mathbb R$, FMG suggests the following form of a conditional generator matching loss
\[L_{\mathrm{cgm}}(\theta) =  \mathbb E_{t, Z\sim p_Z, X_t\sim p_{t|Z}(x|Z)}[D_{X_t}(F_t^Z(X_t), F_t^\theta(X_t))],\]
and up to a constant in $\theta$, this loss function coincides with the generator matching loss 
\[L_{\mathrm{gm}}(\theta) =  \mathbb E_{t, X_t\sim p_t}[D_{X_t}(F_t(X_t), F_t^\theta(X_t))],\]
so, their minimums also coincide.

We suggest that the class of generator matching loss functions be broadened to include those that have i) explicitly time- and state dependent linear parameterizations, ii) time- and state dependent Bregman divergences, and iii) a nearly arbitrarily weighted time distribution $t\sim \mathcal D$. We also extend this characterization to Edit Flows \citep{editflowscite}, which falls outside of the scope of GM. 
\section{Previous work} 
 The theoretical aspects of time-dependent loss re-weightings for flow and diffusion models are discussed in \citep{kingma2023variationaldiffusionmodels}, \citep{song2021maximumlikelihoodtrainingscorebased}, \citep{Lipman2024FlowMG}, \citep{esser2024scalingrectifiedflowtransformers}, and others. We do not attempt an exhaustive survey here and instead try to clarify this in a general setting.
 
\section{Result}
\subsection{Time and state varying linear parametrizations}\label{sec:timedeplinparam}
We extend the notion of a linear parametrization \citep{gmcite} to explicitly vary with time and state, whereas in \cite{Lipman2024FlowMG} it was already clarified to be able to vary with state, but to the best of our knowledge not with both. 

Following their notation, let $S$ be our state space, $\mathcal T$ a class of test functions on $S$ and let $B(S)$ be the bounded functions on $S$. Fix a time $t\in[0,1]$ and $x\in S$. A time and state varying linear parametrization of the subspace $W_t \subset \{ \mathcal L : \mathcal T \to B(S) \mid  \mathcal L \text{ is linear}\}$ is given by: i) a closed, convex set $\Omega_{t,x} \subset V_{t,x} $, where $(V_{t,x}, \langle\cdot,\cdot\rangle_{V_{t,x}})$ forms an inner product space; and ii) a linear operator $\mathcal K_{t,x}: \mathcal T \to V_{t,x}$ such that every $\mathcal L_t \in W_t$ can be written as 
 \[\mathcal L_tf(x) = \langle \mathcal K_{t,x}f, \  F_{t}(x) \rangle_{V_{t,x}},\]
 for a function $F_t: S \to \bigsqcup_{x\in S} \Omega_{t,x}$ such that $F_t(x) \in \Omega_{t,x}$ for every time $t$ and $x\in S$. If such a construction exists, then $F_t$ is said to linearly parametrize $W_t$, or with the appropriate context, to linearly parametrize the generator $\mathcal L_t$. 
 
 Assume that for each $z \in \mathcal Z$, where $\mathcal Z$ is a latent conditioning space, the conditional Markov process $X_t^z$ has a conditional infinitesimal generator $\mathcal L_t^z$. Then a time and state varying linear parametrization of $\mathcal L_t^z$ is given by
 \[\mathcal L_t^z f(x) = \langle \mathcal K_{t,x}f, \  F_t^z(x) \rangle_{V_{t,x}},\]
 where $F_{t}^z: S\to \bigsqcup_{x\in S}\Omega_{t,x}$ is such that $F_t^z(x) \in \Omega_{t,x}$ for every time $t$, $x\in S$ and $z\in \mathcal Z$, and we require the following regularity conditions:
\begin{enumerate}
    \item \label{assump:linparam} For each time $t\in [0,1]$ and $x\in S$, it holds that $\dim (V_{t,x}) < \infty$ and \[\mathbb E_{Z\sim p_{Z|t}(dz|x)}\|F_t^Z(x)\|_{V_{t,x}}<\infty,\]
    where $p_{Z|t}(dz|x)$ is the posterior distribution $(Z|X_t=x)$. 
 \item\label{assump:modelgm} The GM regularity assumptions listed in Appendix~\ref{app:gm} hold for the marginal and model parametrized generators and their associated Markov processes. For the model parametrized generators, we write
 \[\mathcal L_t^\theta f(x) = \langle \mathcal K_{t,x}f, F_t^\theta(x)\rangle_{V_{t,x}},\]
 and we denote $X_t^\theta$ for the associated Markov process. For the marginal generator, we write 
 \[\mathcal L_tf(x) = \langle \mathcal K_{t,x} f, F_t(x) \rangle_{V_{t,x}}\]
 and simply denote $X_t$ for the associated Markov process. See Theorem~\ref{thm:prop1} for a connection between the conditional and marginal parametrizations. 
\end{enumerate} 
\begin{remark}
    We believe that the theory mentioned here can be extended to when, e.g., $V_{t,x}$ is a Hilbert space, and that in this case similar conditions hold. However, for simplicity, we consider only when $\dim V_{t,x} < \infty$, because we need to parametrize the vectors in $V_{t,x}$ by a neural network $F_t^\theta(x)$. 
\end{remark}
 \subsection{Time and state varying loss}\label{section:loss}
In this section, we extend the concept of a generator matching loss (GM loss) and a conditional generator matching loss (CGM loss) as introduced in GM and FMG, to allow for a general class of time distributions $\mathcal D$, with an extended class of Bregman divergences $D_{t,x}$ that depend on both time and state, applied to linear parameterizations that are time and state dependent. 

As in e.g. \cite{JMLR:v6:banerjee05b} and \cite{bregmanart}, we consider Bregman divergences corresponding to $\phi: \Omega\to \mathbb R$ that might only differentiable on the relative interior of $\Omega$, denoted $\mathrm{ri}(\Omega)$, which is defined as the interior of $\Omega$ taken within its affine hull (see Remark~\ref{remark:bregex}). This is in contrast to requiring differentiability on all of $\Omega$. In particular, this is useful for rigorously defining the Poisson-like Bregman divergence and the Binary Cross Entropy Bregman divergence in Examples~\ref{ex:bregpoiss} and~\ref{ex:bregbce}. To our knowledge, Bregman divergences were only defined in association with convex functions differentiable on all of $\Omega$ in GM and FGM, making it difficult to rigorously reason about Bregman divergences associated with convex functions that are not differentiable on the boundary. 

For each time $t\in [0,1]$ and $x\in S$, let
$\Omega_{t,x}$ be a closed convex subset of the inner product space
$(V_{t,x},\langle\cdot,\cdot\rangle_{V_{t,x}})$ and let $\phi_{t,x}:\Omega_{t,x}\to\mathbb R$ be a strictly convex function such that either: (i) $\phi_{t,x}$ is differentiable on all of $\Omega_{t,x}$, in which case we set $\widehat\Omega_{t,x} := \Omega_{t,x}$; or, (ii) $\phi_{t,x}$ is continuous on $\Omega_{t,x}$ and differentiable on $\mathrm{ri}(\Omega_{t,x})$, in which case we set $\widehat\Omega_{t,x} := \mathrm{ri}(\Omega_{t,x})$, where $\mathrm{ri}(\Omega_{t,x})$ is the relative interior of $\Omega_{t,x}$, as seen in Remark~\ref{remark:bregex}.

For $\phi_{t,x}$ of type (i) and (ii), we define the Bregman divergence $D_{t,x}: \Omega_{t,x} \times \widehat \Omega_{t,x} \to \mathbb R$ of type (i) resp. (ii) associated to $\phi_{t,x}$ by 
\[
  D_{t,x}(a,b)= \phi_{t,x}(a) - \phi_{t,x}(b) - \langle a-b,\nabla\phi_{t,x}(b)\rangle_{V_{t,x}}.
\] 
For Bregman divergences with $\widehat\Omega_{t,x} = \mathrm{ri}(\Omega_{t,x})$, a closure property is required that, in practice, is nonrestrictive (cf. Assumption~\ref{assump:int3} and its succeeding remark). Note that whenever $\Omega_{t,x} = V_{t,x}$, so that the loss function is defined on the whole space, types (i) and (ii) coincide. 

The class of time distributions $\mathcal D$ that we consider are probability measures on $[0,1]$ that dominate the Lebesgue measure $\lambda$, denoted $\mathcal D \gg \lambda$ (see Remark~\ref{remark:measuredom} for further clarification). The class of weighting functions $w(t)$ are i) nonnegative on $[0,1]$ and such that $\mathbb E_{t\sim \mathcal D}[w(t)] < \infty$; and, ii) positive except for a set of Lebesgue measure zero. Note that reweighting the loss by such a $w(t)$ is tantamount to taking the expectation over an alternative time distribution which also dominates the Lebesgue measure and selecting an alternative Bregman divergence rescaled by a constant factor, as discussed in Lemma~\ref{lemma:timeweight}. 

We extend the notion of a GM loss to allow for 
\[L_{\mathrm{gm}}(\theta) = \mathbb E_{t\sim \mathcal D, X_t\sim p_t(dx)}[w(t)D_{t,X_t}(F_t(X_t), F_t^\theta(X_t))] \]
and we similarly extend the CGM loss to allow 
\[L_{\mathrm{cgm}}(\theta) = \mathbb E_{t \sim \mathcal D, X_t\sim p_t(dx), Z\sim p_{Z|t}(dz|x)}[w(t)D_{t,X_t}(F_t^Z(X_t), F_t^\theta(X_t))].\]
In the above,  we require the following integrability conditions: 
\begin{enumerate}
    \item \label{assump:int1} $\mathbb E_{t\sim \mathcal D, X_t\sim p_t(dx)}[w(t)D_{t,X_t}(F_t(X_t), F_t^\theta(X_t))]  < \infty$,
    \item \label{assump:int2} $\mathbb E_{t\sim \mathcal D, X_t\sim p_t, Z \sim p_{Z|t}(dz|x)}w(t)[D_{t,X_t}(F_t^Z(X_t), F_t(X_t))] < \infty,$
    \item\label{assump:int3} $\mathbb E_{Z\sim p_{Z|t}(dz)}[F_t^Z(x)] \in \widehat \Omega_{t,x}$ and $F_t^\theta(x) \in \widehat \Omega_{t,x}$. 
\end{enumerate}
\begin{remark}
    The third condition is automatically true if $\widehat\Omega_{t,x} = \Omega_{t,x}$, since $\Omega_{t,x}$ is convex and closed. For $\widehat \Omega_{t,x} = \mathrm{ri}(\Omega_{t,x})$, the condition is not particularly restrictive. Indeed, if it does not hold, then the essential convex hull of $z\mapsto F_t^z(x)$ lies in the boundary of $\Omega_{t,x}$, as noted in the case when $\mathcal Z $ is finite in \cite{JMLR:v6:banerjee05b}. If this condition doesn't hold exactly, then you can only ever be able to approach the true expectation, and this assumption is implicit in e.g. parametrizing probabilities via a sigmoid function in discrete models --- in which case, the model cannot ever emit a probability of exactly $0$ or $1$ with finite logits.  
\end{remark}

\begin{remark}\label{remark:measuredom}
The condition that $\mathcal  D $ dominates $\lambda$, i.e., $\mathcal D \gg \lambda$, is equivalent to the condition that for any $A \subset [0,1]$ with $\lambda(A) > 0$, it also holds $\mathcal D(A) >0$. In other words, the time distribution must assign positive mass to any subset of $[0,1]$ with positive Lebesgue measure. Equivalently, we could have said that $\lambda$ is absolutely continuous with respect to $\mathcal D$.
\end{remark}
\begin{remark}
    A property is said to hold $\lambda$-almost everywhere if the set at which it does not hold has $\lambda$-measure zero. The condition that $w(t)$ is positive except for a set of Lebesgue measure zero could have also been stated as $w(t)$ being positive $\lambda$-almost everywhere. 
\end{remark}
\begin{example}
    It e.g. holds that $\mathcal D \gg \lambda$ if $\mathcal D$ has a density that is positive except for countably many points.
\end{example}
\begin{example}
The density could, for example, be zero on either or both sides of the boundary, as for a $\mathrm{Beta}(\alpha,\beta)$ distribution. 
\end{example}
\subsection{Time dependent loss reweighting}
Note: In what follows, as in GM, we denote 
\[\langle \mu ,f \rangle = \int f(x) \mu(dx)\]
for the duality pairing between probability measures $\mu$ on $S$ and test functions $f\in \mathcal T$.

\begin{lemma}\label{lemma:timeweight}

    Let $\mathcal D$ be any probability measure on $[0,1]$ dominating the Lebesgue measure, i.e. $\mathcal D \gg \lambda$, and let $w(t) \geq 0 $ satisfy $\mathbb E_{t\sim \mathcal D}[w(t)] < \infty$ and $w(t) >0$ for $\lambda$-almost every $t\in [0,1]$. Then 
    \[\widetilde {\mathcal D}(dt) := \frac{w(t)}{\int w(t) \mathcal D(dt)} \mathcal D(dt)\]
    is a probability measure on $[0,1]$ such that $\widetilde D\gg \lambda$, and it holds
    \[\mathbb E_{t\sim \mathcal D} [w(t) f(t)] = K\cdot \mathbb E_{t\sim \widetilde{\mathcal{D}}}[f(t)] \]
    where $K = \int w(t) \mathcal D(dt) >0$.
\end{lemma}
\begin{proof}
    Fix $A \subset [0,1]$ with $\lambda(A) > 0$. Since $w(t) >0$ for $\lambda$-almost every $t\in [0,1]$, we have \[\lambda(A \cap \{w(t) >0\}) >0 ,\] as $\lambda(A\cap \{w(t) =0\}) \leq \lambda(\{w(t) =0\}) = 0$ and thus
    \begin{align*}
        \lambda(A) &=\lambda(A\cap \{w(t) >0\}) + \lambda(A\cap \{w(t)=0\}) = \lambda(A\cap \{w(t) >0\}).
    \end{align*}
     This implies that  $\mathcal D(A\cap \{w(t)>0\}) >0$, since $\mathcal D \gg \lambda$. It follows that there must be an $n_0 \in \mathbb N$ such that $\mathcal D(A\cap \{w(t) > n_0^{-1}\}) > 0 $, since 
    \begin{align*}
    0 < \mathcal D(A\cap \{w(t)>0\}) &= \mathcal D\left(\bigcup_{n=1}^\infty A\cap \{w(t) >n^{-1}\}\right)\\
    &\leq \sum_{n=1}^\infty \mathcal D(A \cap \{w(t) > n^{-1}\}),
    \end{align*}
    and if there were no such $n_0$, then each term in the sum on the right hand side would be zero. As we have assumed $w(t) \geq 0$ for all $t\in [0,1]$, we can write 
    \begin{align*}
        \int_A w(t) \mathcal D(dt) &=  \int_{A\cap \{w(t) >0\}} w(t) \mathcal D(dt)+\int_{A\cap \{w(t) = 0\}} w(t) \mathcal D(dt)\\
        &=\int_{A\cap \{w(t)> 0\}} w(t) \mathcal D(dt),
    \end{align*}
    noting that the second integral in the first equality vanishes, since $w(t) = 0$ for every $t\in A \cap \{w(t) = 0\}$ . Now, through the bound $\mathcal D(A\cap \{w(t) > n_0^{-1}\}) >0$, we obtain
    \begin{align*}
       \int_A w(t) \mathcal D(dt) =  \int_{A\cap \{w(t) >0\}} w(t) \mathcal D(dt) &\geq \int_{A\cap \{w(t) > n_0^{-1}\}}w(t) \mathcal D(dt) \\
        &\geq \int_{A\cap \{w(t) > n_0^{-1}\}}n_0^{-1} \mathcal D(dt) \\
        &= n_0^{-1} \mathcal D(A\cap \{w(t) > n_0^{-1}\}) \\
        &>0,
    \end{align*}
    which shows that $\int_A w(t) \mathcal D(dt) >0$ whenever $\lambda(A) >0$. Let $K := \int_{[0,1]} w(t) \mathcal D(dt)$.  Then since $\lambda([0,1]) = 1 > 0$, it follows that $K >0$. Also note that $K < \infty$, since by assumption $\int w(t) \mathcal D(dt) = \mathbb E_{t\sim \mathcal D} [w(t)] < \infty.$ Hence,
    \[\widetilde{\mathcal D}(A) = K^{-1} \int_A w(t) \mathcal D(dt) > 0\]
    for all $A\subset [0,1]$ with $\lambda( A) > 0$. Hence, $\widetilde{D} \gg \lambda$, and we conclude by noting that 
    \begin{align*} 
    \mathbb E_{t\sim \widetilde{\mathcal D}}[f(t) ] = \int f(t) \widetilde{\mathcal D}(dt) &= K^{-1}\int f(t) w(t) \mathcal D(dt)   \\
    &=  K^{-1} \cdot  \mathbb E_{t\sim \mathcal D}[w(t)f(t) ].
    \end{align*}
\end{proof}
\begin{theorem}\label{thm:timedistribution}
    Let $L_\mathrm{gm}(\theta)$ be the generator matching loss defined in Section~\ref{section:loss},\[L_{\mathrm{gm}}(\theta) = \mathbb E_{t\sim \mathcal D, X_t\sim p_t(dx)}[w(t)D_{t,X_t}(F_t(X_t), F_t^\theta(X_t))] \]
    where the time distribution $\mathcal D$ dominates the Lebesgue measure $\lambda$, i.e. $\mathcal D \gg \lambda$, and the weighting function $w(t)$ is positive $\lambda$-almost everywhere. Suppose that \[L_\mathrm{gm}(\theta) = 0\]and that the GM regularity assumptions listed in Appendix~\ref{app:gm} hold for $\mathcal L_t$ and $\mathcal L_t^\theta$, for the same class of test functions $\mathcal T$. Then the model parametrized generator $\mathcal L_t^\theta$ solves the Kolmogorov Forward Equation (KFE) for $p_t$ on $[0,1]$, and the Markov process $X_t^\theta$ associated to $\mathcal L_t^\theta$ with the initial distribution $p_0$ will satisfy $X_1^\theta \sim p_1(dx)$.
\end{theorem}
\begin{proof}
    Since $\mathcal D \gg \lambda$ and $w(t) >0$ for $\lambda$-almost every $t\in [0,1]$, we use Lemma~\ref{lemma:timeweight} to obtain $\widetilde D \gg \lambda$ and a constant $K >0$ such that 
    \begin{align*}
    L_\mathrm{gm}(\theta) &= \mathbb E_{t\sim \mathcal D, X_t \sim p_t(dx)}[w(t)D_{t,X_t}(F_t(X_t), F_t^\theta(X_t))] \\
    &= \mathbb E_{t\sim \mathcal D}[w(t) \mathbb E_{X_t\sim p_t(dx)}[D_{t,X_t}(F_t(X_t),F_t^\theta(X_t))]\\
    &= K^{-1}\cdot  \mathbb E_{t\sim \widetilde {\mathcal D}} [\mathbb E_{X_t \sim p_t(dx)}[D_{t,X_t}(F_t(X_t), F_t^\theta(X_t))] \\
    &= K^{-1} \cdot \mathbb E_{t\sim \widetilde D, X_t\sim p_t(dx)}[ D_{t,X_t}(F_t(X_t),F_t^\theta(X_t))].
    \end{align*}
    Since rescaling the loss by a positive constant factor simply corresponds to a different choice of Bregman divergence (cf. Example~\ref{ex:timebreg}), it therefore suffices to only consider the case $w(t) \equiv 1$ and $\mathcal D \gg \lambda$ in the proof going forward.

    Suppose that $L_\mathrm{gm}(\theta) =0 $. For each $t\in [0,1]$ and $f\in \mathcal T$, consider the set $E_t := \{ x \in S: \mathcal L_t^\theta f(x) = \mathcal L_t f(x)\}$, and let $A \subset [0,1]$ be the set of all $t\in [0,1]$ such that $p_t(S\setminus E_t) = 0$. 

    Since Bregman divergences are non-negative, $L_\mathrm{gm}(\theta) = 0$ forces \[D_{t,X_t}(F_t(X_t), F_t^\theta(X_t)) = 0\]to hold $p_t$-almost surely for $\mathcal D$-almost every $t \in [0,1]$, and thus $\lambda$-almost every $t\in [0,1]$ since $\mathcal D \gg \lambda$. By Corollary~\ref{cor:breg}(c), \[D_{t,x}(F_t(x), F_t^\theta(x)) = 0  \iff F_t(x) = F_t^\theta(x),\]
    and therefore $\mathcal L_t^\theta f(X_t) = \mathcal L_tf(X_t) $ holds $p_t $-almost surely for $\lambda$-almost every $t\in [0,1]$, from which it follows that $\lambda(A) = 1$. Now consider $h_f^\theta: [0,1]\to \mathbb R$, defined by  \[h_f^\theta(t) := \langle p_t, \mathcal L_t^\theta f\rangle - \langle p_t, \mathcal L_t f\rangle.\] Then, for every $t\in A$ we have
    
    \begin{align*} 
    h_f^\theta(t) = \langle p_t , \mathcal L_t^\theta f - \mathcal L_t f\rangle &= \int_S [\mathcal L_t^\theta f(x) - \mathcal L_t f(x)] p_t(dx) \\
    &= \int_{E_t} [\mathcal L_t^\theta f(x) - \mathcal L_t f(x)] p_t(dx) + \int_{S\setminus E_t} [\mathcal L_t^\theta f(x) - \mathcal L_t f(x)]p_t(dx) \\
    &= 0 ,
    \end{align*}
    where in the last equality, we have used that $\mathcal L_t^\theta f(x) - \mathcal L_tf(x) = 0$ for all $x\in E_t$ and that $p_t(S\setminus E_t) = 0$. 

    By the argument of Theorem~\ref{thm:dualcont}, applied to $\mathcal L_t$ and $\mathcal L_t^\theta$, both $t\mapsto \langle p_t, \mathcal L_t f\rangle $ and $t\mapsto \langle p_t, \mathcal L_t^\theta f\rangle $ are continuous on $[0,1]$, and hence $h_f^\theta(t)$ is continuous on $[0,1]$. 

    Since the GM regularity assumptions are assumed to hold for $\mathcal L_t$ and $\mathcal L_t^\theta$, we can apply Theorem~\ref{thm:dualcont} to $\mathcal L_t$ and $\mathcal L_t^\theta$ to see that $h_f^\theta(t)$ is continuous. Since $h_f^\theta(t) = 0$ for $\lambda$-almost every $t\in [0,1]$, it follows that $h_f^\theta(t) = 0$ for all $t\in [0,1]$ by continuity. 
    
    This implies that for all $t\in [0,1]$, the KFE holds
    \begin{align*} 
    \langle p_t , \mathcal L_t^\theta f\rangle&= \langle p_t, \mathcal L_t f\rangle \\
    &=  \partial_t \langle p_t, f\rangle
    \end{align*}
    which by Assumption~\ref{assump:gm5} implies that $X_1^\theta \sim p_1(dx)$. 
\end{proof}
\begin{example}
    Suppose $\mathcal D$ has a density against the Lebesgue measure that is positive Lebesgue-almost everywhere in $[0,1]$. That is, it would permissible if its density was zero on a set of Lebesgue-measure zero. Then $\mathcal  D \gg \lambda$ since $A\subset [0,1]$ and $\lambda(A) > 0$ implies $\mathcal D(A) > 0$, so we may take our loss to be weighted by $\mathcal D$. That is, it is permissible for $\mathcal D$ to have a probability density function on $[0,1]$ that is positive except for a set of Lebesgue-measure zero. 
\end{example}
\begin{example}
        It would, for example, be permissible if $\mathcal D$ had a density that vanished at finitely many points of $[0,1]$, so e.g. either or both boundary points of $[0,1]$, e.g., a $\mathrm{Beta}(\alpha,\beta)$ distribution. It would also be permissible if $\mathcal D$ had a density that vanished at countably many points of $[0,1]$. 
\end{example}

 \subsection{Generator Matching}
 Assume for every $z\in \mathcal Z$ there corresponds a conditional Markov process $X_t^z$ with conditional infinitesimal generator $\mathcal L_t^z$ and linear parametrization 
 \[\mathcal L_t^zf = \langle \mathcal K_{t,x}f; F_t^z(x)\rangle_{V_{t,x}},\]
  as is described in Section~\ref{sec:timedeplinparam}. By Proposition 1 in GM, we have that
 \[\mathcal L_t f(x) = \mathbb E_{Z \sim p_{Z|t}(dz| x)}[\mathcal L_t^Zf(x)]\]
 generates $p_t(dx) = \int_{\mathcal Z} p_t(dx|z)p_{Z}(dz)$ --- i.e., the KFE holds: 
\[\partial_t \langle p_t, f\rangle = \langle p_t, \mathcal L_t f\rangle .\]

The proof of Proposition 1 in GM Appendix C.2 for time and state varying linear parametrizations runs through the same and is adapted below:
 \begin{theorem}\label{thm:prop1}
    Let \[ \mathcal L_t^z f(x) = \langle \mathcal K_{t,x} f; F_t^z(x)\rangle_{V_{t,x}}\]
be a linear parametrization of $\mathcal L_t^z$, as in Section~\ref{sec:timedeplinparam}, for $F_t^z(x) \in \Omega_{t,x} \subset V_{t,x}$ convex and closed. Then it follows that
\[F_t(x) := \mathbb E_{Z\sim p_{Z|t}(dz|x)}[F_t^Z(x)]\]
linearly parametrizes the marginal generator.
\end{theorem}
\begin{proof}
    We have
    \begin{align*}
        \partial_t \langle p_t(dx), f(x)\rangle &= \mathbb E_{X_t\sim p_t(dx), Z\sim p_{Z|t}(dz|x)}[\mathcal L_t^Zf(X_t)] \\
        &= \mathbb E_{X_t\sim p_t(dx)}\mathbb E_{Z\sim p_{Z|t}(dz|x)}[\langle \mathcal K_{t,X_t}f, \ F_t^Z(X_t)\rangle_{V_{t,X_t}}] \\
        &= \mathbb E_{X_t\sim p_t(dx)} [\langle \mathcal K_{t,X_t}f , \ \mathbb E_{Z\sim p_{Z|t}(dz|x)}F_t^Z(X_t)\rangle_{V_{t,X_t}}] \\
        &= \left\langle p_t(dx), \Big\langle \mathcal K_{t,x}f , \ \underbrace{\mathbb E_{Z\sim p_{Z|t}(dz|x)}F_t^Z(x)}_{=:F_t(x)}\Big\rangle_{V_{t,x}}\right\rangle,
    \end{align*}
    where we have used Lemma~\ref{lemma:linparamreg} in the third equality, possible by Assumption~\ref{assump:linparam} in Section~\ref{sec:timedeplinparam}. This shows that
    \[\mathcal L_tf(x) := \langle \mathcal K_{t,x}f(x),\ F_t(x)\rangle_{V_{t,x}}\]
    solves the KFE for $p_t$. We also note that $F_t(x) \in \Omega_{t,x}$, since $\Omega_{t,x}$ is convex and closed and $F_t(x)$ is the limit of convex combinations of elements in $\Omega_{t,x}$.
\end{proof}
Next, we adapt Proposition 2 from GM to support a time and state varying Bregman divergence $D_{t,x}$ with a corresponding time and state varying linear parametrization.
\begin{theorem}\label{thm:prop2}
    Let $F_t^z: S\to \bigsqcup_{x\in S}\Omega_{t,x}$ be a time and state varying linear parametrization of the conditional generator as in Section~\ref{sec:timedeplinparam},  from which it holds that 
    \[F_t(x) = \mathbb E_{Z\sim p_{Z|t}(dz|x)}[F_t^Z(x)]\]
    linearly parametrizes the marginal generator, by Theorem~\ref{thm:prop1}. Let $D_{t,x}: \Omega_{t,x} \times \widehat{\Omega}_{t,x}\to \mathbb R$ be a Bregman divergence for each $t \in [0,1]$ and each $x\in S$. Then the CGM loss and the GM loss coincide up to a constant in $\theta$
    \[\nabla_\theta \mathbb E_{t\sim \mathcal D, X_t \sim p_t(dx), Z\sim p_{Z|t}(dz|x)}[w(t)D_{t,X_t}(F_t^Z(X_t), F_t^\theta(X_t))] = \nabla_\theta \mathbb E_{t\sim \mathcal D, X_t\sim p_t(dx)}[w(t)D_{t,X_t}(F_t(X_t), F_t^\theta(X_t))]\]
    where $\mathcal D \gg \lambda$ and $w(t) >0$ for $\lambda$-almost every $t\in [0,1]$, under the integrability conditions in Section~\ref{section:loss}.
\end{theorem}
\begin{proof}
By the remarks in the beginning of the proof of Theorem~\ref{thm:timedistribution}, it suffices to consider when $w(t) \equiv 1$ and $\mathcal D \gg \lambda$, since the reweighting of the loss by $w(t)$ merely corresponds to an alternative time distribution $\widetilde {\mathcal D} \gg \lambda$ with a rescaled choice of Bregman divergence. Now, note that by the second integrability condition in Section~\ref{section:loss}, it follows that 
\[\mathbb E_{Z \sim p_{Z|t}(dz|X_t)}[D_{t,X_t}(F_t^Z(X_t), F_t(X_t))]< \infty \]
holds $p_t$-almost surely for $\mathcal D $-almost every $t\in [0,1]$. Therefore, we may use Corollary~\ref{cor:breg}(b) to obtain that
\[\mathbb E_{Z\sim p_{Z|t}(dz|X_t)}[D_{t,X_t}(F_t^Z(X_t), F_t^\theta(X_t))] = D_{t,X_t}(F_t(X_t), F_t^\theta(X_t))  + \mathbb E_{Z \sim p_{Z|t}(dz|X_t)} [D_{t,X_t}(F_t^Z(X_t), F_t(X_t))].\] 
holds $p_t$-almost surely for $\mathcal D$-almost every $t\in[0,1]$. Now, we can write
     \begin{align*}
    &\mathbb E_{t\sim \mathcal D, X_t \sim p_t(dx), Z\sim p_{Z|t}(dz|X_t)}[D_{t,X_t}(F_t^Z(X_t), F_t^\theta(X_t))]\\
    &\quad=  \mathbb E_{t\sim \mathcal D, X_t\sim p_t(dx)} \Big[\mathbb E_{Z\sim p_{Z|t}(dz|X_t)}[D_{t,X_t}(F_t^Z(X_t), F_t^\theta(X_t))]\Big]\\
     &\quad= \mathbb E_{t\sim \mathcal D, X_t\sim p_t(dx)} \Big[D_{t,X_t}(F_t(X_t), F_t^\theta(X_t))  + \mathbb E_{Z \sim p_{Z|t}(dz|x)} D_{t,X_t}(F_t^Z(X_t), F_t(X_t))\Big]\\
     &\quad= \mathbb E_{t\sim \mathcal D, X_t\sim p_t(dx)} [D_{t,X_t}(F_t(X_t), F_t^\theta(X_t))] + \underbrace{\mathbb E_{t\sim \mathcal D, X_t\sim p_t(dx),Z\sim p_{Z|t}(dz|X_t)}[D_{t,X_t}(F_t^Z(X_t),F_t(X_t))]}_{\mathrm{const.}}
    \end{align*}
    from which the result follows.
\end{proof}

\subsection{Sums of linear parametrizations}\label{sec:sumlinparam}
Fix a time $t$ and $x\in S$. Suppose that for all $f\in \mathcal T$ the infinitesimal generator $\mathcal L_t$ can be written as a sum of linear parametrizations 
\begin{align*}
    \mathcal L_tf(x) = \sum_{i=1}^{N_{t,x}} \langle \mathcal K_{t,x}^i f, F_t^i(x)\rangle_{V_{t,x}^i}
\end{align*}
for $N_{t,x} < \infty$ and a function $F_t^i(x) \in \Omega_{t,x}^i \subset V_{t,x}^i$, where $\Omega_{t,x}^i$ is convex and closed, and a linear operator $\mathcal K_{t,x}^i : \mathcal T \to V_{t,x}^i$, for $i=1,\ldots, N_{t,x}$.

In this section, we show that a valid per-term conditional generator matching loss is given by a loss whose per-term contribution is a sum of Bregman divergences, one along each component. This was mentioned in GM Proposition 5 in the context of multimodal generative models with factorized probability paths, but we would like to make this note more general.

We rewrite the sum of linear parametrizations as a single linear parametrization, and then we take our Bregman divergence to be separable. 

Define $V_{t,x}:= \bigoplus_{i=1}^{N_{t,x}} V_{t,x}^i$ and $\Omega_{t,x} := \prod_{i=1}^{N_{t,x}} \Omega_{t,x}^i$. Then $\Omega_{t,x} \subset V_{t,x}$ is convex and closed, since it is the product of convex and closed sets, and we set $F_t(x) := (F_t^i(x))_{i=1}^{N_{t,x}} \in \Omega_{t,x}$. Let $v = (v^1,\ldots, v^{N_{t,x}})$ and $w = (w^1,\ldots, w^{N_{t,x}})$ be elements of $V_{t,x}$. An inner product on $V_{t,x}$ is given by 
\[\langle v, w\rangle_{V_{t,x}} = \sum_{i=1}^{N_{t,x}} \langle v^i, w^i \rangle_{V_{t,x}^i},\]
and the map $\mathcal K_{t,x}f = (\mathcal K_{t,x}^1f,\ldots, \mathcal K_{t,x}^{N_{t,x}}f)$ is a linear operator from $\mathcal T$ to $V_{t,x}$. A linear parametrization of the generator $\mathcal L_t$ is therefore given by 
\[\mathcal L_tf(x) = \langle \mathcal K_{t,x}f, F_t(x) \rangle_{V_{t,x}}.\]
Assume the conditions in in Section~\ref{sec:timedeplinparam} and \ref{section:loss} hold for the overall linear parametrization. Consider the separable Bregman divergence $D_{t,x}: \Omega_{t,x}\times\widehat \Omega_{t,x} \to \mathbb R$ acting on each component by way of an individual Bregman divergence $D_{t,x}^i: \Omega_{t,x}^i \times \widehat{\Omega}_{t,x}^i\to \mathbb R$, so that
\[D_{t,x}(y,x) = \sum_{i=1}^{N_{t,x}} D_{t,x}^i(y^i,x^i).\]
Note that in the above, $\widehat \Omega_{t,x} = \prod_{i=1}^n \widehat\Omega_{t,x}^i$, and to be in full accordance with our notion of a Bregman divergence, we should have $\widehat\Omega_{t,x}^i = \Omega_{t,x}^i$ for $i=1\ldots, N_{t,x}$ or $\widehat \Omega_{t,x}^i = \mathrm{ri}(\Omega_{t,x}^i)$ for $i=1,\ldots, N_{t,x}$. See Example~\ref{ex:bregsep} and the succeeding remark for further details. A valid conditional generator matching loss is thus given by
\begin{align*}
L_\mathrm{cgm}(\theta) &= \mathbb E_{t\sim \mathcal D, Z\sim p_Z, X_t \sim p_t(dx|z)}[D_{t,X_t}( F_t^Z(X_t), F_t^\theta(X_t)) ]\\
&=\mathbb E_{t\sim \mathcal D, Z\sim p_Z, X_t \sim p_t(dx|z)}\left[\sum_{i=1}^{N_{t,X_t}} D_{t,X_t}^i(F_t^{Z,i}(X_t),F_t^{\theta,i}(X_t))\right].
\end{align*}

\subsection{Flow matching $X_1$ prediction}\label{section:flow}
 The \emph{endpoint prediction} formulation of flow matching losses provides a natural way to obtain a time-dependent linear parameterization of the drift, and is discussed in e.g.  \cite{zhang2025trajectoryflowmatchingapplications,Lipman2024FlowMG, Oresten2025.11.03.686219}, and others.

In flow matching, conditioned on terminating at the endpoint $x_1\in \mathbb R^n$, a time-dependent vector field $u_t(x|x_1)$ governs the drift along the conditional paths. Suppose that $u_t(x|x_1)$ is affine in $x_1$, i.e. \[u_t(x|x_1) = A_{t,x} x_1 + b_{t,x},\]
with $A_{t,x} \in \mathbb R^{n\times n}$ and $b_{t,x} \in \mathbb R^n$. We write the infinitesimal generator corresponding to this vector field as
\begin{align*} 
\mathcal L_t^{x_1}f(x) &= \nabla f(x)^T  (A_{t,x}x_1 + b_{t,x})\\
        &= \langle A_{t,x}^T\nabla f(x),  x_1\rangle_{\mathbb R^n} + \langle b_{t,x}^T\nabla f(x), 1\rangle_{\mathbb R}.
\end{align*}
The above is a sum of linear parametrizations as in Section~\ref{sec:sumlinparam}. The model parametrization can be written 
\begin{align*}
    \mathcal L_t^{\theta}f(x) &= \nabla f(x)^T  (A_{t,x}\hat x_1^\theta(t,x) + b_{t,x}) \\
        &= \langle A_{t,x}^T\nabla f(x),  \hat x_1^\theta(t,x)\rangle_{\mathbb R^n} + \langle b_{t,x}^T\nabla f(x), 1\rangle_{\mathbb R},
\end{align*}
with a corresponding CGM loss:
\[
L_\mathrm{cgm}(\theta) = \mathbb E_{t\sim \mathcal D, X_1\sim p_Z, X_t\sim p_t(\cdot|x_1)}[w(t)D_{t,X_t}(X_1,\hat x_1^\theta(t, X_t))]
\]
where $\hat{x}_1^\theta(t,X_t)$ denotes the endpoint predicted by the model, and $w(t) >0$ for $\lambda$-almost every $t\in[0,1]$. 

\begin{example}
    One could take
    \[u_t(x|x_1)= \frac{x_1 - x}{1-t}\]
    in the above.
\end{example} 
\begin{example}\label{ex:flow1}
    Let $u_t(x|x_1) = A_{t,x} x_1 + b_{t,x}$ be as above and parametrize $x_1$-predictions by $\hat x_1^\theta(t,X_t)$. Due to the linear parametrization 
    \[\mathcal L_t^{x_1}f(x) = \langle A_{t,x}^T\nabla f(x),  x_1\rangle_{\mathbb R^n} + \langle b_{t,x}^T\nabla f(x), 1\rangle_{\mathbb R},\]
    a valid CGM loss is given by
    \[
    L_{\mathrm{cgm}}(\theta) = \mathbb E_{t\sim \mathcal D, X_1 \sim p_Z, X_t\sim p_t(dx |x_1)} \left[w(t)\|X_1 - \hat X_1^\theta(t,X_t)\|^2\right],
    \]
    for some reweighting function $w(t) $ that is positive $\lambda$-almost everywhere on $[0,1]$, using the Bregman divergence $D(y,x) := \|y-x\|^2$ (cf. Example~\ref{ex:msebreg}).
\end{example}
\begin{example}\label{ex:flow2}
Fix $\varepsilon >0$, and let $c(t) = \dfrac{1}{(1-t+\varepsilon)^2}$. Then a valid CGM loss is given by 
    \begin{align*}
    L_{\mathrm{cgm}}(\theta) &= \mathbb E_{t\sim \mathcal D, X_1 \sim p_Z, X_t\sim p_t(dx |x_1)} \left[\frac{\|X_1 - \hat x_1^\theta(t,X_t)\|^2}{(1-t+\varepsilon)^2}\right]\\
    &=  \mathbb E_{t\sim \mathcal D, X_1 \sim p_Z, X_t\sim p_t(dx |x_1)} \left[\left\|\frac{X_1 - \hat x_1^\theta(t,X_t)}{1-t+\varepsilon} \right\|^2\right].
    \end{align*}
    This is in contrast with the usual conditional flow matching loss under $x_1$ prediction
    \[L_{\mathrm{cfm}}(\theta) = \mathbb E_{t\sim \mathcal D, X_1 \sim p_Z, X_t\sim p_t(dx |x_1)}\left[\left\|\frac{X_1 - \hat x_1^\theta(t,X_t)}{1-t}\right\|^2\right]\]
    whose per-term contribution has a singularity at $t=1$.
\end{example}

\subsection{Diffusion models and $x_0$-prediction}
As in GM Section H.2, to view denoising diffusion models from the perspective of generator matching, we consider a Markov noising process $\bar X_t$ given by $d\bar X_t = \sigma_t d\bar W_t$ from $t=1$ to $t=0$ backwards in time. Then the KFE holds in reverse time with
\[\partial_t \langle p_t,f\rangle = \langle p_t, \nabla f^T\frac{\sigma_t^2}{2} \nabla \log p_t\rangle\]
as it is shown in GM, corresponding to the probability flow ODE in \cite{song2021scorebasedgenerativemodelingstochastic}. We write and linearly parametrize the conditional generator as follows: 
\begin{align*} 
\mathcal L_t^{z} f(x) &= \nabla f(x)^\top \frac{\sigma_t^2}{2}\nabla_x \log p_t(x|z)\\
&= \langle \frac{\sigma_t^2}{2} \nabla f(x), \nabla_x \log p_t(x|z)\rangle_{\mathbb R^n}.
\end{align*}
With this parametrization, an analogue of Example~\ref{ex:flow2} thus holds for diffusion models:
\begin{example}
    Suppose that $u_t(x|x_0) := \nabla_x\log p_t(x|x_0) = A_{t,x} x_0 + b_{t,x}$ is affine in $x_0$. Then a valid CGM loss is given by 
    \[L(\theta) = \mathbb E_{t\sim \mathcal D, x_0 \sim p_\mathrm{data}, x_t \sim p_t(\cdot|x_0)}[w(t) \|x_0 - \hat x_0^\theta(t,x_t)\|^2]\]
    for a reweighting function $w(t)$ that is positive $\lambda$-almost everywhere on $[0,1]$. 
\end{example}

\subsection{Rescaling time-dependent constants out of jump models}

As in GM, we let $Q_t(dy;x)$ specify a time-dependent jump kernel on the state space $S$. It can be useful to model jump intensities in terms of continuous hazards $h_{t,j}(x) > 0$, together with individual rate multipliers $R_{t,j}(x) \geq 0 $ that are continuous, and to train a model to match the rate multiplier (cf. Equation 28 and Proposition 5 in \cite{gat2024discreteflowmatching}, where in their notation this is done by predicting the posterior $\hat w_t^j(x^i|X_t)$ and the hazards are $a_t^{i,j}$). 

Fix a time $t$ and $x\in S$. As in Branching Flows \citep{nordlinder2025branchingflowsdiscretecontinuous}, consider a time-dependent jump kernel of the form: 
\[Q_t(dy; x) = \sum_{j=1}^{N_{t,x}} h_{t,j}(x)R_{t,j}(x) \delta_{\Gamma_{t,j}(x)}(dy),\]
for continuous functions $\Gamma_{t,j}: S\to S$ that we call jump targets, and continuous rate multipliers $R_{t,j}(x)\geq 0$, for $j=1,\ldots, N_{t,x}$. We define
\[\lambda_\mathrm{total}(t,x) = \int_S Q_t(dy;x) = \sum_{j=1}^{N_{t,x}} h_{t,j}(x) R_{t,j}(x) \]
for each $x\in S$ and $t\in [0,1]$, and we require that $\lambda_\mathrm{total}(t,x) < \infty$. 
The associated process $X_t$ can be described by: 
\[X_{t+\Delta t} = \begin{cases}
    X_t &\quad\text{with probability} \quad 1 - \Delta t\lambda_\mathrm{total}(t,x) + o(\Delta t)\\
    \sim J_t(dy;x) &\quad\text{with probability} \quad \Delta t\lambda_{\mathrm{total}}(t,x) + o(\Delta t),
\end{cases}\]
where in the above, the distribution $J_t(dy;x)$ is the normalized jump distribution:
\[J_t(dy; x) = \frac{Q_t(dy;x) }{\lambda_{\mathrm{total}}(t,x)}= \frac{1}{\lambda_{\mathrm{total}}(t,x)}\sum_{j=1}^{N_{t,x}}  h_{t,j}(x)R_{t,j}(x) \delta_{\Gamma_{t,j}(x)}(dy)\]
having infinitesimal generator  
\[\mathcal L_tf(x) = \int(f(y) - f(x) ) Q_t(dy;x) = \sum_{j=1}^{N_{t,x}} (f(\Gamma_{t,j}(x)) - f(x)) h_{t,j}(x)R_{t,j}(x). \]
Write \[R_t(x) = (R_{t,j}(x))_{j=1}^{N_{t,x}}.\]A time and state dependent linear parametrization of the generator is given by 
\[\mathcal L_tf(x) = \langle \mathcal K_{t}f(x), R_{t}(x)\rangle_{V_{t,x}}, \]
where $V_{t,x} = \mathbb R^{N_{t,x}}$ is equipped with the inner product
\[\langle v,w \rangle_{V_{t,x}} = \sum_{j=1}^{N_{t,x}}h_{t,j}(x) v^jw^j,\]
and the linear map $\mathcal K_{t,x} : \mathcal T \to \mathbb R^{N_{t,x}}$ sends \[\mathcal K_{t,x}f = (f(\Gamma_{t,i}(x)) - f(x))_{i=1}^{N_{t,x}}.\]

\begin{example}
    In accordance with the above, for latent $z\in \mathcal Z$ we specify a conditional atomic time-dependent jump kernel through conditional rate multipliers $R_{t,j}^z(x)$:
    \[Q_t^z(dy; x) = \sum_{j=1}^{N_{t,x}}  h_{t,j}(x)R_{t,j}^z(x) \delta_{\Gamma_{t,j}(x)}(dy),\]
    making the same assumptions on $h_{t,j}(x)$, $R_{t,j}^z$, $\lambda_\mathrm{total}^z(t,x) := \sum_{j=1}^{N_{t,x}} h_{t,j}(x) R_{t,j}^z(x)$ and on the marginal rates $R_{t,j}(x)$ and $\lambda_\mathrm{total}(t,x)$ as above. In particular, we have 
    \[\lambda_\mathrm{total}(x) = \mathbb E_{Z\sim p_{Z|t}(dz|x)} [\lambda_\mathrm{total}^z(t,x)] < \infty.\]
    Using the elementary inequality $\sum_i a_i^2 \leq (\sum_i a_i)^2$ for $a_i \geq 0$, we can show that $\mathbb E_{Z\sim p_{Z|t}(dz|x)}\|R_{t}^Z(x)\|_{V_{t,x}} < \infty$, since  
    \begin{align*}
    \|R_t^z(x) \|_{V_{t,x}} = \left(\sum_{j=1}^{N_{t,x}} h_{t,j}(x) R_{t,j}^z(x)^2\right)^{1/2} &\leq \sum_{j=1}^{N_{t,x}} \sqrt{h_{t,j}(x)} R_{t,j}^z(x) \\ 
    &\leq C_{t,x} \lambda_{\mathrm{total}}^z(t,x)
    \end{align*}
    where $C_{t,x} = \max_{j}\{ I_{h_{t,j}(x) > 0} \cdot h_{t,j}(x)^{-1/2}\}$ is independent of $z$, so it follows that 
    \[\mathbb E_{Z\sim p_{Z|t}(dz|x)}\|R_{t}^Z(x)\|_{V_{t,x}} \leq C_{t,x} \mathbb E_{Z\sim p_{Z|t}(dz|x)}[\lambda_\mathrm{total}^z(x)] < \infty.\]
    This satisfies the condition for Lemma~\ref{lemma:linparamreg}, so a valid CGM loss under the conditions in Section~\ref{section:loss} is given by 
    \[
    L_\mathrm{cgm}(\theta) = \mathbb E_{t\sim \mathcal D, Z\sim p_Z(dz), X_t\sim p_t(dx|z)}[D_{t,X_t}(R_t^Z(X_t), R_t^\theta(X_t))].
    \]
    In particular, note that the hazard rates $h_{t,j}(x)$ are not present in the loss above, so the loss can be taken directly against $X_1$-predictions (e.g., in DFM, the posterior $\hat w_t^j(x^i|X_t)$).
\end{example}

\begin{remark}
    In accordance with Assumption~\ref{assump:gm2} in Appendix~\ref{app:gm}, the expected number of jumps should be finite. This corresponds to
\[\mathbb E\left[\int_0^1\lambda_\mathrm{total}(t,X_t)dt\right]< \infty\]
in the above.
\end{remark}
\begin{example}
    As used in Branching Flows \citep{nordlinder2025branchingflowsdiscretecontinuous}, suppose the model predicted rate multipliers are $\rho_{t,i}^\theta \in (0,1)$ and the ground truth conditional rate multipliers are $\rho_{t,i}^z \in [0,1]$. A valid CGM loss is
    \begin{align*}
    L_{\mathrm{cgm}}(\theta) &= \textstyle \mathbb E_{t\sim \mathcal D, Z\sim p_Z, X_t\sim p_t(dx| z)}[ \sum_{i=1}^{N_{t,x}}D_{\mathrm{BCE}}(\rho_{t,i}^Z(X_t), \rho_{t,i}^\theta(X_t))]\\
    &=  \textstyle \mathbb E_{t\sim \mathcal D, Z\sim p_Z, X_t\sim p_t(dx| z)}\left[ -\sum_{i=1}^{N_{t,x}}\big(\rho_{t,i}^Z(X_t)\log( \rho_{t,i}^\theta(X_t)) + (1-\rho_{t,i}^Z(X_t))\log(1-\rho_{t,i}^\theta(X_t))\big) + \mathrm{const}.\right]
    \end{align*}
    via the separable Bregman divergence which takes the binary cross entropy loss on each component, considering $0\cdot \log 0 := 0$ in the above (cf. Example~\ref{ex:bregbce}). 
\end{example}
\begin{example}
    Alternatively, also as used in Branching Flows, we may consider the Poisson-style Bregman divergence (cf. Example~\ref{ex:bregpoiss}) between predicted rate multipliers $R_{t,i}^\theta(x) \in (0,\infty)$ and conditional, ground-truth rate multipliers $R_{t,i}^z(x) \in [0,\infty)$,
    \begin{align*}
    L_{\mathrm{cgm}}(\theta) &= \mathbb E_{t \sim \mathcal D, Z\sim p_Z, X_t\sim p_t(dx| z)}\left[\textstyle\sum_{i=1}^{N_{t,x}} D_{\mathrm{Poiss}}(R_{t,i}^Z(X_t), R_{t,i}^\theta(X_t))\right] \\
    &=\mathbb E_{t \sim \mathcal D, Z\sim p_Z, X_t\sim p_t(dx| z)}\left[\textstyle\sum_{i=1}^{N_{t,x}}\big(R_{t,i}^\theta(X_t) - R_{t,i}^Z(X_t) \log(R_{t,i}^\theta(X_t))\big) + \mathrm{const.}\right].
    \end{align*}
\end{example}

\subsection{Edit Flows Propositions}

In the notation of Edit Flows (EF) \citep{editflowscite}, we let $u_t(x,z|x_t,z_t)$ generate $p_t(x,z)$ on $\mathcal  X\times \mathcal  Z.$ By the first part of EF Theorem 3.1, 
\[u_t(x|x_t) := \sum_z \mathbb E_{p_t(z_t|x_t)}u_t(x,z|x_t,z_t)\qquad\text{generates} \qquad p_t(x) := \sum_z p_t(x,z).\]

We show an extended second part of EF Theorem 3.1:
\begin{theorem}\label{thm:EF31}
    For each $t\in [0,1]$ and $x\in \mathcal X$, it holds that 
    \begin{align*} 
    \textstyle \nabla_\theta \mathbb E_{t \sim \mathcal D,x_t, z_t\sim p_t(x,z)} &[w(t) \textstyle D_{t,x_t}(\sum_z a(t)u_t(\cdot,z| x_t,z_t), b(t)u_t^\theta(\cdot|x_t))] \\
    &= \nabla_\theta \mathbb E_{t \sim\mathcal D,x_t \sim p_t}[w(t)D_{t,x_t}(a(t)u_t(\cdot| x_t), b(t)u_t^\theta(\cdot|x_t))],
    \end{align*}
    where $a(t), b(t) >0$ are internally reweighting the loss and $w(t) >0$ is externally reweighting the loss, under the integrability assumptions: 
    \begin{enumerate}
        \item[(a)] $\mathbb E_{t\sim \mathcal D, x_t\sim p_t}[w(t)D_{t,x_t} (a(t) u_t(\cdot|x_t), b(t)u_t^\theta(\cdot|x_t))] < \infty $,
        \item[(b)] $\mathbb E_{t\sim \mathcal D, x_t\sim p_t,z_t\sim p_t(z_t|x_t)} [w(t)D_{t,x_t}(\sum_z a(t)u_t(\cdot,z|x_t,z_t), b(t)u_t(\cdot|x_t))] < \infty.$
    \end{enumerate}
\end{theorem}
\begin{proof}
    Note that
    \[\mathbb E_{z_t \sim p_t(z_t|x_t)}\left[\sum_z a(t)u_t(\cdot,z| x_t,z_t)\right] = a(t)u_t(\cdot| x_t).\]
    The proof then runs analogously to that of Theorem~\ref{thm:prop2}. Using Corollary~\ref{cor:breg}(b) in the second equality, and the integrability assumptions (a) and (b) to ensure well-posedness, we have
    \begin{align*} 
    \textstyle \mathbb E_{t \sim \mathcal D,x_t, z_t\sim p_t(x,z)} [&w(t)\textstyle D_{t,x_t}(\sum_z a(t)u_t(\cdot,z| x_t,z_t), b(t)u_t^\theta(\cdot|x_t))] \\
    &= \textstyle\mathbb E_{t\sim \mathcal D, x_t \sim p_t} \left[w(t)\mathbb E_{z_t \sim p_{t}(z_t|x_t)}[D_{t,x_t}(\sum_z a(t) u_t(\cdot,z| x_t,z_t), b(t) u_t^\theta(\cdot|x_t))]\right]\\
    &=\textstyle \mathbb E_{t\sim \mathcal D, x_t \sim p_t}\Big[w(t)D_{t,x_t}(a(t)u_t(\cdot| x_t), b(t)u_t^\theta(\cdot|x_t))  \\
    &\textstyle \qquad\qquad \qquad  + \mathbb E_{z_t\sim p_t(z_t|x_t)}\left[w(t)D_{t,x_t}(\sum_z a(t) u_t(\cdot,z| x_t,z_t), b(t)u_t(\cdot|x_t))\right]\Big] \\
    &= \mathbb E_{t\sim \mathcal D, x_t \sim p_t}[w(t)D_{t,x_t}(a(t)u_t(\cdot| x_t), b(t)u_t^\theta(\cdot|x_t))] \\ 
    &\textstyle \qquad +\underbrace{\textstyle\mathbb E_{t\sim \mathcal D, x_t\sim p_t, z_t\sim p_t(z_t|x_t)}[w(t)D_{t,x_t}(\sum_z a(t) u_t(\cdot,z| x_t,z_t), a(t)u_t(\cdot|x_t))]}_{\mathrm{const.}}
    \end{align*}
    from which the result follows.
\end{proof}

Note that Theorem~\ref{thm:EF31} recovers the original statement of the second part of EF Theorem 3.1 when for each $t$ and $x_t$, we define $w(t) \equiv a(t) \equiv  b(t) \equiv 1$, $D_{t,x_t} \equiv D$ and $\mathcal D = \delta_{t}$ for some choice of Bregman divergence $D$. We also remark that scaling both $a(t)$ and $b(t)$ by the same factor has the same effect as scaling the loss by $w(t)$.
    \begin{example}\label{ex:4.9}

Let $h(t)$ be some overall hazard rate, e.g. \[h(t) = \frac{\dot\kappa_t}{1-\kappa_t}\] for a scheduler $\kappa_t$. Recasting the per-time contribution to the loss in terms of a rescaled prediction $ v_t^\theta(x|x_t) $ of $\frac{u_t(x|x_t)}{h(t)}$, we have:
\[\nabla_\theta \mathbb E_{x_t, z_t\sim p_t(x,z)} [\textstyle D(\tfrac{1}{h(t) } \sum_z u_t(\cdot,z|x_t,z_t), v_t^\theta(\cdot|x_t))] = \nabla_\theta \mathbb E_{x_t\sim p_t}D(\tfrac{1}{h(t) }u_t(\cdot|x_t) ,  v_t^\theta(\cdot|x_t))\]
for any Bregman divergence $D$. The loss minimum is found at
\[v_t^{\theta^*}(\cdot|x_t) := \frac{1}{h(t)} u_t(\cdot|x_t),\]
implying that $v_t^\theta(\cdot|x_t)$ can be used to parametrize edit rates via
\[u_t^\theta(\cdot|x_t) := h(t)v_t^\theta(\cdot |x_t),\]
whereby we recover the marginal  rates at the loss minimum:
\[u_t^{\theta^*}(\cdot|x_t) = h(t) v_t^{\theta^*}(\cdot|x_t) = u_t(\cdot|x_t).\]
\end{example}

\appendix
\section{Properties of Bregman Divergences}
We state and prove a more general analogue of the statement of Proposition 1 in \cite{JMLR:v6:banerjee05b}.
\begin{lemma}\label{lemma:bregnori}
    Let $(V,\langle\cdot,\cdot\rangle)$ be a finite dimensional inner product space, let $\Omega \subset V$ be convex and closed, and let $\phi: \Omega \to \mathbb R $ be a strictly convex differentiable function. Define the Bregman divergence $D_\phi: \Omega\times \Omega \to \mathbb R$ by
    \[D_\phi(a,b) = \phi(a) - \phi(b) - \langle a-b, \nabla \phi(b)\rangle . \] 
    Let $(\Xi, \mathcal F, \mu)$ be a probability space and let $X: \Xi \to \Omega \subset V$ be a random vector such that 
    \begin{enumerate}
        \item $\mathbb E_\mu \|X\| < \infty$,
        \item $\mathbb E_\mu[D_\phi(X, \mathbb E_\mu [X])] < \infty$.
    \end{enumerate} 
    Then it holds that
    \begin{enumerate} 
    \item[(a)] $\mathbb E_{\mu}[X]$ is the unique minimizer of $y\mapsto \mathbb E_{\mu}[D_\phi(X,y) ]$ over $\Omega$.
    \item[(b)] $\mathbb E_{\mu}[D_\phi(X,y)] = D_\phi(\mathbb E_\mu[X], y)+\mathbb E_\mu [D_\phi(X,\mathbb E_\mu[X])] $, for all $y\in \Omega$.
    \item[(c)] $D_\phi(x,y) = 0 \iff x=y$ and $D_\phi(x,y) > 0$ if $x\neq y$, over $x,y \in \Omega$. 
    \end{enumerate}
\end{lemma}
\begin{proof}
        The condition in (c) is the first order characterization of strict convexity of $\phi$, as seen in Equation (3.3) in \cite{citeulike:163662}. For (b), first note that $\mathbb E_\mu[X] \in \Omega$ since $\Omega$ is convex and closed and $\mathbb E_\mu[X]$ is the limit of convex combinations in $\Omega$. A simple computation gives
    \begin{align*}
        D_\phi(x,y) - D_\phi(z,y) &=  \phi(x) - \phi(z) - \langle x-z, \nabla \phi(y)\rangle\\
        &= D_\phi(x,z) - \langle x-z, \nabla\phi(y) - \nabla\phi(z)\rangle.
    \end{align*}
    It follows that
    \begin{align*}
        \mathbb E_\mu \big[D_\phi(X,y) - D_\phi(\mathbb E_\mu[X],y)\big] 
        &=\mathbb E_\mu\big[D_\phi(X,\mathbb E_\mu[X])) - \langle X- \mathbb E_\mu[X], \ \nabla\phi(\mathbb E_\mu[X]) - \nabla \phi(y)\rangle\big]\\
        &= \mathbb E_\mu\big[D_\phi(X,\mathbb E_\mu[X])\big] - \mathbb E_\mu\big[\langle X-\mathbb E_\mu[X], \ \nabla \phi(\mathbb E_\mu[X]) - \nabla \phi(y)\rangle\big]\\
        &= \mathbb E_\mu\big[D_\phi(X,\mathbb E_\mu[X])\big] - \langle \mathbb E_\mu[X- \mathbb E_\mu [X]], \ \nabla \phi(\mathbb E_\mu[X]) - \nabla\phi(y)\rangle \\
        &= \mathbb E_\mu [D_\phi(X,\mathbb E_\mu[X])] \\
    \end{align*}
    Noting that $D_\phi(\mathbb E_\mu[X],y)< \infty$ and $\mathbb E_\mu [D_\phi(X,\mathbb E_\mu[X])]< \infty$, we can conclude (b). Finally, (a) follows from (c) and (b).
\end{proof}
\begin{remark}\label{remark:bregex}
    Following \cite{Lee01111986}, the relative interior of $\Omega$, denoted $\mathrm{ri}(\Omega)$, is the interior of $\Omega$ taken within its affine hull. The affine hull of $\Omega$, denoted $\mathrm{aff}(\Omega)$, is the set of affine combinations of elements of $\Omega$, and an affine combination of $x_1,\ldots, x_n \in \Omega$ is given by $\sum_{i=1}^n \alpha_i x_i,$ where $\alpha_1,\ldots, \alpha_n \in \mathbb R$ are such that $\sum_{i=1}^n \alpha_i = 1$. Note that in contrast to convex combinations, the terms $\alpha_i $ can be negative. For example, if $\Omega$ contains two distinct points, then $\mathrm{aff}(\Omega)$ must contain the line through them.
\end{remark}
\begin{lemma}\label{lemma:bregri}
    Let $\Omega \subset V$ be a closed convex subset of $(V,\langle\cdot,\cdot\rangle)$, and let $\phi: \Omega \to \mathbb R$ be a continuous strictly convex function that is differentiable in $\mathrm{ri}(\Omega)$. Define the Bregman divergence $D_\phi: \Omega\times \mathrm{ri}(\Omega) \to \mathbb R$ by
    \[D_\phi(a,b) = \phi(a) - \phi(b) - \langle a-b, \nabla\phi(b)\rangle.\]
    Let $(\Xi, \mathcal F, \mu)$ be a probability space and let $X: \Xi \to \Omega \subset V$ be a random vector such that $\mathbb E_\mu[X]\in \mathrm{ri}(\Omega)$. Then under the same integrability conditions as Lemma~\ref{lemma:bregnori}, it holds
    \begin{enumerate}
        \item[(a)] $\mathbb E_\mu[X]$ is the unique minimizer of $y\mapsto \mathbb E_\mu [D_\phi(X,y)]$ over $\mathrm{ri}(\Omega)$.
        \item[(b)] $\mathbb E_{\mu}[D_\phi(X,y)] = D_\phi(\mathbb E_\mu[X], y)+\mathbb E_\mu [D_\phi(X,\mathbb E_\mu[X])] $, for all $y\in \mathrm{ri}(\Omega)$.
        \item[(c)] $D_\phi(x,y) = 0 \iff x=y$ and $D_\phi(x,y) > 0$ if $x\neq y$, over $x\in \Omega$ and $y\in \mathrm{ri}(\Omega)$.
    \end{enumerate}
\end{lemma}
\begin{proof}
    (c) is equivalent to  Property 1 in Appendix A of \cite{JMLR:v6:banerjee05b}. (b) follows in entirely the same manner as in Lemma~\ref{lemma:bregnori}, and (a) follows from (b) and (c).
\end{proof}
\begin{corollary}\label{cor:breg}
    Let $\Omega$ be convex and closed, and let $\phi: \Omega\to \mathbb R$ be a strictly convex continuous function, differentiable on $\widehat \Omega \in \{ \mathrm{ri}(\Omega), \Omega\}$, as in Section~\ref{section:loss}, and define the Bregman divergence $D_\phi: \Omega\times \widehat \Omega \to \mathbb R$ by
    \[D_\phi(x,y) = \phi(x) - \phi(y) - \langle x-y, \nabla \phi(y)\rangle.\]
    Let $(\Xi, \mathcal F, \mu)$ be a probability space and let $X: \Xi\to \Omega \subset V$ be a random vector such that 
    \begin{enumerate}
        \item $\mathbb E_\mu\|X\| < \infty$, and if $\widehat \Omega = \mathrm{ri}(\Omega)$, then $\mathbb E_\mu [X] \in \mathrm{ri}(\Omega). $
        \item $\mathbb E_\mu[D_\phi(X,\mathbb E_\mu(X))] < \infty$.
    \end{enumerate}
    Then it holds 
    \begin{enumerate}
        \item[(a)] $\mathbb E_\mu[X] $ is the unique minimizer of $y\mapsto \mathbb E_\mu[D_\phi(X,y)]$ over $\widehat \Omega$,
        \item[(b)] $\mathbb E_\mu[D_\phi(X,y)] = D_\phi(\mathbb E_\mu[X], y)  + \mathbb E_\mu[D_\phi(X,\mathbb E_\mu[X])]$ for all $y\in \widehat \Omega$,
        \item[(c)] $D_\phi(x,y) = 0 \iff x=y$ and $D_\phi(x,y) >0$ if $x\neq y$ over $x\in \Omega$ and $y\in \widehat \Omega$.
    \end{enumerate}
\end{corollary}
\begin{proof}
    If $\widehat \Omega = \Omega$, then this is Lemma~\ref{lemma:bregnori}. If $\widehat \Omega = \mathrm{ri}(\Omega)$, then this is Lemma~\ref{lemma:bregri}.
\end{proof}

\section{Examples of Bregman Divergences}\label{app:bregex}

Here, we present several examples of Bregman divergences, e.g., as in \cite{gmcite} and \cite{Lipman2024FlowMG}. 

\begin{example}[Time-Scaled Bregman Divergence]\label{ex:timebreg}
     Let $\phi_t: \Omega_t\to \mathbb R$ be a strictly convex differentiable function from the closed convex set $\Omega_t \subset V_t$ where $(V_t, \langle\cdot,\cdot\rangle_t)$ is an inner product space. Let $w(t) >0$ and consider the mapping $\psi_t(x)  = w(t) \phi_t(x)$. Under these conditions, 
    \begin{align*}
    D_{\psi_t}(y,x) &= \psi_t(y) - \psi_t(x) - \langle y-x, \nabla \psi_t(x)\rangle  \\
    &= w(t)\Big(\phi_t(y) - \phi_t(x) - \langle y-x, \nabla \phi_t(x) \rangle\Big) \\
    &= w(t) D_{\phi_t}(y,x).
    \end{align*}
    The above shows that $w(t) D_{\phi_t}(y,x)$ defines a Bregman divergence for each $t\in[0,1]$. A critical example is when $w(t) \equiv C > 0 $, some positive constant, which in particular shows that a Bregman divergence rescaled by a positive constant is still a Bregman divergence.
\end{example}
\begin{example}[MSE]\label{ex:msebreg}
    Let $\phi_\mathrm{MSE}: \mathbb R^N\to \mathbb R$ be the strictly convex and differentiable function $\phi_\mathrm{MSE}(x) = \|x\|^2 = \sum_{i=1}^Nx^ix^i$, using superscripts for components. Then the Bregman divergence associated to $\phi_\mathrm{MSE}$ is given by 
    \begin{align*}
    D_{\mathrm{MSE}}(y,x) &= \phi_\mathrm{MSE}(y) - \phi_\mathrm{MSE}(x) - \langle y-x, \nabla \phi_\mathrm{MSE}(x)\rangle \\
    &= \sum_{i=1}^N( y^i y^i - x^ix^i - 2(y^i-x^i)(x^i)) =\|y-x\|^2. \end{align*}
\end{example}
\begin{example}[Separable Bregman Divergences]\label{ex:bregsep}
    Let $\phi_i: \Omega_i \to \mathbb R$ be strictly convex and differentiable functions for $i=1,\ldots, N$. Suppose further that $\Omega_i \subset V_i$ is convex and closed, where $(V_i, \langle \cdot,\cdot\rangle_i) $ is an inner product space. Consider $V = \bigoplus_{i=1}^N V_i$ and $\Omega = \prod_{i=1}^N \Omega_i$, and let the inner product on $V$ be given by $\langle y, x\rangle = \sum_{i=1}^N \langle y^i, x^i\rangle_i$, using superscripts to denote the $i$'th-coordinate. 
    
    Notice that $\Omega$ is convex and closed, and define $\Phi: \Omega \to \mathbb R$ by $\Phi(x) = \sum_{i=1}^N \phi_i(x^i)$. Then $\Phi$ is strictly convex and differentiable. This gives rise to a \emph{separable} Bregman divergence, where $D_{\phi_1},\ldots, D_{\phi_N}$ act separately on each of the $N$-components:
    \begin{align*}
    D_\Phi(y,x) &= \Phi(y) - \Phi(x) - \langle y-x, \nabla \Phi(x)\rangle \\
    &=  \sum_{i=1}^N \Big(\phi_i(y^i) - \phi_i(x^i) - \langle y^i - x^i, \nabla \phi_i(x^i)\rangle_{i} \Big)\\
    &=\sum_{i=1}^N D_{\phi_i}(y^i,x^i).
    \end{align*}
\end{example}
\begin{remark}
    There is a natural analogue to Example~\ref{ex:bregsep} when each Bregman divergence acts on the relative interior via $D_{\phi_i}: \Omega_i \times \mathrm{ri}(\Omega_i) \to \mathbb R$, since $\mathrm{ri}(\prod_{i=1}^n \Omega_i) = \prod_{i=1}^n \mathrm{ri}(\Omega_i)$, noting that each acting on the relative interior should not be particularly restrictive on the class of loss functions since $\mathrm{ri}(V_i) = V_i$. 
\end{remark}
\begin{remark}
    In the next two examples, we show two Bregman divergences up to a constant in the $y$, i.e. the left slot, since the model is passed to the right slot. We also consider Bregman divergences as mappings $D: \Omega\times \mathrm{ri}(\Omega) \to  \mathbb R$ in these examples, as discussed in Section~\ref{section:loss} and Lemma~\ref{lemma:bregri}. Note that $\mathrm{ri}(\Omega)$ is the interior of $\Omega$ within the affine hull of $\Omega$ --- this is discussed in more detail in Remark~\ref{remark:bregex}.
\end{remark}
\begin{example}[Poisson-Style Bregman]\label{ex:bregpoiss}
    Let $\phi_\mathrm{Poiss}: [0,\infty)\to \mathbb R$ be given by 
    \[\phi_\mathrm{Poiss}(x) = \begin{cases}
        x \log x & \text{if} \quad x\in (0,\infty),\\
        0 & \text{if} \quad x =0.
    \end{cases}\]
    A computation shows that $\phi$ is strictly convex and continuous on $[0,\infty)$ and differentiable on $\mathrm{ri}([0,\infty)) = (0,\infty)$. Its associated Bregman divergence is $D_\mathrm{Poiss}: [0,\infty)\times (0,\infty)\to \mathbb R$,
    \[
    D_{\mathrm{Poiss}}(y,x) =y\log y - x\log x - (y-x)(1+\log x) =  x - y\log x + f(y),
    \]
    where $f(y) = y \log y -y$ is a function only of $y$. 
\end{example}
\begin{example}[Binary Cross Entropy] \label{ex:bregbce}
    Let $\phi_\mathrm{BCE}: [0,1] \to \mathbb R$ be given by 
    \[\phi_\mathrm{BCE}(x) = \begin{cases}
        x\log x + (1-x)\log(1-x) &\text{if} \quad x\in (0,1) \\
        0 &\text{if} \quad x\in \{0,1\}.
    \end{cases}\]
    Then $\phi$ is strictly convex and continuous on $[0,1]$ and differentiable on $\mathrm{ri}([0,1]) = (0,1)$. Its associated Bregman divergence is $D_\mathrm{BCE}: [0,1] \times (0,1) \to \mathbb R$,
    \begin{align*}
    D_{\mathrm{BCE}}(y,x) &= D_{\mathrm{Poiss}}(y,x) + D_{\mathrm{Poiss}}(1-y,1-x)\\
    &= -[y\log x + (1-y)\log (1-x)] + g(y),
    \end{align*}
    where $g(y) = y\log (y) + (1-y)\log(1-y)$.
\end{example}

\section{Regularity conditions inherited from Generator Matching}\label{app:gm}

We repeat a few basic definitions from Generator Matching (GM) \citep{gmcite} and their regularity conditions. The infinitesimal generator $\mathcal L_t$ of a process $X_t$ acts on a test function $f \in \mathcal T$ via 
\[\mathcal L_t f(x) = \frac{d}{dh}\Bigg|_{h=0} \Big(\mathbb E[f(X_{t+h})|X_t = x]\Big),\]
and analogously, the conditional infinitesimal generator $\mathcal L_t^z$ of a conditional Markov process $X_t^z$ is defined for each $z\in \mathcal Z$ by 
\begin{align*}
\mathcal L_t^z f(x) &= \frac{d}{dh}\Bigg|_{h=0}\Big(\mathbb E[f(X_{t+h}^z)|X_t^z = x] \Big).
\end{align*}
 Also following GM, we denote 
\[\langle \mu, f\rangle := \mathbb E_{x\sim \mu}[f(x)] = \int f(x) \mu(dx)\]
for the duality pairing between a Borel probability measure $\mu$ on our state space $S$ (assumed to be a Polish metric space) and a test function $f \in \mathcal T \subset C_0(S)$, denoting $C_0(S)$ for the functions that vanish at infinity. That is, $f\in C_0(S)$ if for all $\varepsilon >0$ there is a compact $K$ such that $|f(x)| < \varepsilon$ for all $x\in S\setminus K$. In accordance with GM Section A.2., we make the following regularity assumptions:
\begin{enumerate}
    \item\label{assump:gm1} The Markov process $(X_t)_{0\leq t\leq 1}$ associated to $\mathcal L_t$ is Feller, in the sense defined in GM Appendix~A.1.2. In particular, we require that $\mathcal L_t f\in C_0(S)$ for all $f\in \mathcal T$. 
    \item\label{assump:gm2} In each time interval $[s,t]$, the expected number of discontinuities of $u\mapsto X_u$ is finite.
    \item\label{assump:gm3} There is a dense subspace $\mathcal T \subset C_0(S)$ satisfying i) $\mathcal T \subset \mathrm{dom}(\mathcal L_t)$, and ii) the function $t\mapsto \mathcal L_tf$ is continuous for any $f\in \mathcal T$. Moreover, two probability distributions $\mu_1$ and $\mu_2$  are equal if and only if $\mathbb E_{x\sim \mu_1}[f(x)] = \mathbb E_{x\sim \mu_2} [f(x)]$ holds for all $f\in \mathcal T$. 
    \item \label{assump:gm4} Any probability path $(p_t)_{0\leq t \leq 1}$ satisfies that $t\mapsto \langle p_t, f\rangle $ is continuous in $t$ for all $f\in \mathcal T$. 
    \item \label{assump:gm5} If $(p_t)_{0\leq t\leq 1}$ is a probability path on $S$ and $X_t$ is the Markov process associated to $\mathcal L_t$, then  
    \[X_0 \sim p_0 \quad\text{and}\quad \partial_t \langle p_t, f \rangle = \langle p_t, \mathcal L_t f \rangle \implies X_t \sim p_t.\]
\end{enumerate}

\begin{theorem}\label{thm:dualcont}
    Under the GM regularity assumptions found above, for every $f\in \mathcal T$ the map $t\mapsto \langle p_t, \mathcal L_t f\rangle$ is continuous on $[0,1]$.
\end{theorem}
\begin{proof}
    First we show that $t\mapsto \langle p_t,g\rangle$ is continuous for all $g\in C_0(S)$. By Assumption~\ref{assump:gm3}, there is a sequence $(g_n)_{n=1}^\infty \subset \mathcal T$ such that $\|g_n -g \|_\infty \to 0$ as $n\to \infty$, and we see that 
    \begin{align*}
    \sup_{t\in [0,1]}|\langle p_t, g_n \rangle - \langle p_t, g\rangle| =  \sup_{t\in [0,1]}\left|\int (g_n(x)- g(x))p_t(dx)\right| &\leq \|g_n -g \|_\infty  \\
    &\to 0, \text{ as } n\to \infty.
    \end{align*}
    This shows that $t\mapsto \langle p_t, g_n\rangle $ converges uniformly to $t\mapsto \langle p_t, g\rangle$. Since $t\mapsto \langle p_t,g_n\rangle$ is continuous for all $n\in \mathbb N$ and continuity is preserved by uniform limits, we have that $t\mapsto \langle p_t, g\rangle$ is continuous for any fixed $g\in C_0(S)$. Now fix $f\in \mathcal T $ and let $t_n \to t$ in $[0,1]$. Using the triangle inequality, we can bound
    \begin{align*}
    |\langle p_{t_n}, \mathcal L_{t_n} f\rangle - \langle p_t, \mathcal L_t f\rangle | &= |\langle p_{t_n}, \mathcal L_{t_n} f - \mathcal L_t f + \mathcal L_tf\rangle -\langle p_t, \mathcal L_t f\rangle \rangle|\\
    &= |\langle p_{t_n}, \mathcal L_{t_n} f - \mathcal L_tf\rangle + \langle p_{t_n} - p_t, \mathcal L_tf\rangle|\\
    &\leq |\langle p_{t_n}, \mathcal L_{t_n} f - \mathcal L_t f\rangle| + |\langle p_{t_n} - p_t, \mathcal L_tf\rangle|.
    \end{align*}
    Now, we have that $\mathcal L_{t_n} f\to \mathcal L_{t} f$ as $n\to\infty$ by Assumption~\ref{assump:gm3}, so that
    \[|\langle p_{t_n}, \mathcal L_{t_n} f - \mathcal L_t f\rangle| \leq  \|\mathcal L_{t_n} f- \mathcal L_tf\|_\infty \to 0, \text{ as } n\to\infty. \]
    Moreover, for all $f\in \mathcal T$ it holds $\mathcal L_t f\in C_0(S)$ since $X_t$ is Feller by Assumption~\ref{assump:gm1}, so we have that \[ |\langle p_{t_n} - p_t, \mathcal L_tf\rangle | = |\langle p_{t_n}, \mathcal L_t f\rangle - \langle p_t, \mathcal L_t f\rangle |\to 0,\text{ as } n\to\infty\] since we showed the map $t\mapsto \langle p_t, g\rangle$ was continuous for all $g\in C_0(S)$ (which implies that $\langle p_{t_n}, g\rangle \to \langle p_t, g\rangle$ if $t_n \to t$ and $g\in \mathcal T$). Combining the above, we get
    \begin{align*}
    \limsup_{n\to\infty} |\langle p_{t_n}, \mathcal L_{t_n} f\rangle - \langle p_t, \mathcal L_t f\rangle | &\leq \limsup_{n\to\infty} \Big(|\langle p_{t_n}, \mathcal L_{t_n} f - \mathcal L_t f\rangle| + |\langle p_{t_n} - p_t, \mathcal L_tf\rangle|\Big) = 0 ,
    \end{align*}
    so that 
    \[\lim_{n\to\infty} \langle p_{t_n}, \mathcal L_{t_n} f\rangle = \langle p_t, \mathcal L_t f\rangle\]
    and we can conclude that $t\mapsto \langle p_t,\mathcal L_t f\rangle $ is continuous.
\end{proof}
\section{Regularity conditions for linear parametrizations of conditional generators}\label{app:reglinparam}

\begin{lemma}\label{lemma:linparamreg}
    For all times $t$ and all $x\in S$, let $\dim V_{t,x} < \infty$ and let 
    \[ \mathcal L_t^z f(x) = \langle \mathcal K_{t,x} f; F_t^z(x)\rangle_{V_{t,x}}\]
    be a linear parametrization of $\mathcal L_t^z$. Then, whenever $\mathbb E_{Z\sim p_{Z|t}(dz|x)} \|F_t^Z(x)\|_{V_{t,x}} < \infty$, we have 
    \[\mathbb E_{Z\sim p_{Z|t}(dz|x)} [\langle \mathcal K_{t,x}f, F_t^Z(x)\rangle_{V_{t,x}}] = \langle \mathcal K_{t,x}f, \mathbb E_{Z\sim p_{Z|t}(dz|x)}[F_t^Z(x)]\rangle_{V_{t,x}}.\]
\end{lemma}

\begin{proof}
Let $(e_i^{t,x})_{i=1}^{N_{t,x}}$ be an orthonormal basis for $V_{t,x}$ --- such a basis can be obtained by the Gram-Schmidt algorithm. Then writing $v^i := \langle v, e_i^{t,x}\rangle_{V_{t,x}}$ and $w^j := \langle w, e_j^{t,x}\rangle_{V_{t,x}}$, we have:
    \begin{align*} 
    \langle v, w\rangle_{V_{t,x}} &=\textstyle \langle \sum_{i=1}^{N_{t,x}}v^i e_{i}^{t,x}, \sum_{j=1}^{N_{t,x}}w^je_j^{t,x}\rangle  \\
    &= \textstyle \sum_{i=1}^{N_{t,x}}v^i \langle e_i^{t,x}, \sum_{j=1}^{N_{t,x}}w^je_{j}^{t,x}\rangle\\
    &= \textstyle \sum_{i=1}^{N_{t,x}} v^i \sum_{j=1}^{N_{t,x}} w^j \langle e_i^{t,x}, e_j^{t,x}\rangle \\
    &= \textstyle \sum_{i=1}^{N_{t,x}} v^i \sum_{j=1}^{N_{t,x}} w^j \delta_{ij} \\
    &= \textstyle \sum_{i=1}^{N_{t,x}}  v^i w^i.
    \end{align*}
        Denote $(\mathcal K_{t,x}f)_i := \langle \mathcal K_{t,x}f, e_i^{t,x}\rangle_{V_{t,x}},$ and similarly denote $(F_{t}^Z(x))_i := \langle F_t^Z(x), e_i^{t,x}\rangle_{V_{t,x}}. $ Then, it holds 
        \[\mathcal K_{t,x}f = \sum_{i=1}^{N_{t,x}}(\mathcal K_{t,x}f)_ie_i^{t,x} \qquad\text{and} \qquad  F_t^Z(x) = \sum_{i=1}^{N_{t,x}} (F_{t}^Z(x))_ie_i^{t,x}.\]
        By the Cauchy-Schwarz inequality, we can bound 
    \[|(F_{t}^Z(x))_i|\leq \|F_t^Z(x)\| \cdot \|e_i^{t,x}\| = \|F_t^Z(x)\| < \infty.\]
    The $V_{t,x}$-valued expectation is defined by taking the expectation along the components
    \[\mathbb E_{Z\sim p_{Z|t}(dz|x)}[F_t^Z(x)] = \sum_{i=1}^{N_{t,x}}\mathbb E_{Z\sim p_{Z|t}(dz|x)}[(F_t^Z(x))_i]e_i^{t,x}.\]
    In particular, we have
    \[\mathbb E_{Z\sim p_{Z|t}(dz|x)}[F_t^Z(x)]_i := \langle \mathbb E_{Z\sim p_{Z|t}(dz|x)}[F_t^Z(x)], e_i^{t,x}\rangle =  \mathbb E_{Z\sim p_{Z|t}(dz|x)}[(F_t^Z(x))_i] \]
    and we therefore obtain
     \begin{align*}
        \mathbb E[\langle \mathcal K_{t,x}f, F_t^Z(x)\rangle_{V_{t,x}}] &= \textstyle \mathbb E_{Z\sim p_{Z|t}(dz|x)}[\sum_{i=1}^{N_{t,x}}(\mathcal K_{t,x} f)_i \cdot (F_t^Z(x))_i] \\
        &= \textstyle \sum_{i=1}^{N_{t,x}} (\mathcal K_{t,x}f)_i \mathbb E_{Z\sim p_{Z|t}(dz|x)}[F_t^Z(x)_i] \\
        &= \textstyle \sum_{i=1}^{N_{t,x}} (\mathcal K_{t,x}f)_i \mathbb E_{Z\sim p_{Z|t}(dz|x)}[F_t^Z(x)]_i \\
        &= \langle \mathcal K_{t,x}f, \mathbb E_{Z\sim p_{Z|t}(dz|x)}[F_t^Z(x)]\rangle_{V_{t,x}}.
     \end{align*}
     
\end{proof}
\section*{Acknowledgement}
This project received support from the Swedish Research Council (2024-00390 and 2023-02516) and the Knut and Alice Wallenberg Foundation (2024.0039) to B.M.

\bibliographystyle{plainnat}

\bibliography{library}

@Article{Lipman2024FlowMG,
 author = {Y. Lipman and Marton Havasi and Peter Holderrieth and Neta Shaul and Matt Le and Brian Karrer and Ricky T. Q. Chen and David Lopez-Paz and Heli Ben-Hamu and Itai Gat},
 booktitle = {arXiv.org},
 journal = {ArXiv},
 title = {Flow Matching Guide and Code},
 volume = {abs/2412.06264},
 year = {2024}
}

@misc{zhang2025trajectoryflowmatchingapplications,
      title={Trajectory Flow Matching with Applications to Clinical Time Series Modeling}, 
      author={Xi Zhang and Yuan Pu and Yuki Kawamura and Andrew Loza and Yoshua Bengio and Dennis L. Shung and Alexander Tong},
      year={2025},
      eprint={2410.21154},
      archivePrefix={arXiv},
      primaryClass={cs.LG},
      url={https://arxiv.org/abs/2410.21154}, 
}

@misc{song2021scorebasedgenerativemodelingstochastic,
      title={Score-Based Generative Modeling through Stochastic Differential Equations}, 
      author={Yang Song and Jascha Sohl-Dickstein and Diederik P. Kingma and Abhishek Kumar and Stefano Ermon and Ben Poole},
      year={2021},
      eprint={2011.13456},
      archivePrefix={arXiv},
      primaryClass={cs.LG},
      url={https://arxiv.org/abs/2011.13456}, 
}

@book{citeulike:163662,
  author = {Boyd, Stephen and Vandenberghe, Lieven},
  comment = {Full text available online.

---

QA402.5 .B69 2004},
  description = {CiteULike},
  howpublished = {Hardcover},
  isbn = {0521833787},
  month = {March},
  publisher = {{Cambridge University Press}},
  title = {Convex Optimization},
  url = {https://stanford.edu/~boyd/cvxbook/},
  year = 2004
}

@misc{gat2024discreteflowmatching,
      title={Discrete Flow Matching}, 
      author={Itai Gat and Tal Remez and Neta Shaul and Felix Kreuk and Ricky T. Q. Chen and Gabriel Synnaeve and Yossi Adi and Yaron Lipman},
      year={2024},
      eprint={2407.15595},
      archivePrefix={arXiv},
      primaryClass={cs.LG},
      url={https://arxiv.org/abs/2407.15595}, 
}

@misc{nordlinder2025branchingflowsdiscretecontinuous,
      title={Branching Flows: Discrete, Continuous, and Manifold Flow Matching with Splits and Deletions}, 
      author={Hedwig Nora Nordlinder and Lukas Billera and Jack Collier Ryder and Anton Oresten and Aron Stålmarck and Theodor Mosetti Björk and Ben Murrell},
      year={2025},
      eprint={2511.09465},
      archivePrefix={arXiv},
      primaryClass={stat.ML},
      url={https://arxiv.org/abs/2511.09465}, 
}

@article{JMLR:v6:banerjee05b,
  author  = {Arindam Banerjee and Srujana Merugu and Inderjit S. Dhillon and Joydeep Ghosh},
  title   = {Clustering with Bregman Divergences},
  journal = {Journal of Machine Learning Research},
  year    = {2005},
  volume  = {6},
  number  = {58},
  pages   = {1705--1749},
  url     = {http://jmlr.org/papers/v6/banerjee05b.html}
}

@misc{gmcite,
      title={Generator Matching: Generative modeling with arbitrary Markov processes}, 
      author={Peter Holderrieth and Marton Havasi and Jason Yim and Neta Shaul and Itai Gat and Tommi Jaakkola and Brian Karrer and Ricky T. Q. Chen and Yaron Lipman},
      year={2025},
      eprint={2410.20587},
      archivePrefix={arXiv},
      primaryClass={cs.LG},
      url={https://arxiv.org/abs/2410.20587}, 
}

@misc{editflowscite,
      title={Edit Flows: Flow Matching with Edit Operations}, 
      author={Marton Havasi and Brian Karrer and Itai Gat and Ricky T. Q. Chen},
      year={2025},
      eprint={2506.09018},
      archivePrefix={arXiv},
      primaryClass={cs.LG},
      url={https://arxiv.org/abs/2506.09018}, 
}

@misc{nguyen2025oneflowconcurrentmixedmodalinterleaved,
      title={OneFlow: Concurrent Mixed-Modal and Interleaved Generation with Edit Flows}, 
      author={John Nguyen and Marton Havasi and Tariq Berrada and Luke Zettlemoyer and Ricky T. Q. Chen},
      year={2025},
      eprint={2510.03506},
      archivePrefix={arXiv},
      primaryClass={cs.AI},
      url={https://arxiv.org/abs/2510.03506}, 
}

@misc{ho2020denoisingdiffusionprobabilisticmodels,
      title={Denoising Diffusion Probabilistic Models}, 
      author={Jonathan Ho and Ajay Jain and Pieter Abbeel},
      year={2020},
      eprint={2006.11239},
      archivePrefix={arXiv},
      primaryClass={cs.LG},
      url={https://arxiv.org/abs/2006.11239}, 
}

@article{bregmanart,
author = {Della Pietra, Stephen and Dellapietra, Vincent and Lafferty, John},
year = {2002},
month = {01},
pages = {},
title = {Duality and Auxiliary Functions for Bregman Distances}
}

@article{Lee01111986,
author = {Carl W. Lee},
title = {An Introduction to Convex Polytopes. By Arne Brøndsted},
journal = {The American Mathematical Monthly},
volume = {93},
number = {9},
pages = {750--752},
year = {1986},
publisher = {Taylor \& Francis},
doi = {10.1080/00029890.1986.11971939},
URL = { 
        https://doi.org/10.1080/00029890.1986.11971939
},
eprint = { 
    
        https://doi.org/10.1080/00029890.1986.11971939
}
}

@article {Oresten2025.11.03.686219,
	author = {Oresten, Anton and Sato, Kenta and St{\r a}lmarck, Aron and Billera, Lukas and Nordlinder, Hedwig Nora and Ryder, Jack Collier and Kaduk, Mateusz and Murrell, Ben},
	title = {Spontaneous Emergence of Symmetry in a Generative Model of Protein Structure},
	elocation-id = {2025.11.03.686219},
	year = {2025},
	doi = {10.1101/2025.11.03.686219},
	publisher = {Cold Spring Harbor Laboratory},
	URL = {https://www.biorxiv.org/content/early/2025/11/04/2025.11.03.686219},
	eprint = {https://www.biorxiv.org/content/early/2025/11/04/2025.11.03.686219.full.pdf},
	journal = {bioRxiv}
}

@misc{lipmanflowcite,
      title={Flow Matching for Generative Modeling}, 
      author={Yaron Lipman and Ricky T. Q. Chen and Heli Ben-Hamu and Maximilian Nickel and Matt Le},
      year={2023},
      eprint={2210.02747},
      archivePrefix={arXiv},
      primaryClass={cs.LG},
      url={https://arxiv.org/abs/2210.02747}, 
}

@misc{kingma2023variationaldiffusionmodels,
      title={Variational Diffusion Models}, 
      author={Diederik P. Kingma and Tim Salimans and Ben Poole and Jonathan Ho},
      year={2023},
      eprint={2107.00630},
      archivePrefix={arXiv},
      primaryClass={cs.LG},
      url={https://arxiv.org/abs/2107.00630}, 
}

@misc{song2021maximumlikelihoodtrainingscorebased,
      title={Maximum Likelihood Training of Score-Based Diffusion Models}, 
      author={Yang Song and Conor Durkan and Iain Murray and Stefano Ermon},
      year={2021},
      eprint={2101.09258},
      archivePrefix={arXiv},
      primaryClass={stat.ML},
      url={https://arxiv.org/abs/2101.09258}, 
}

@misc{esser2024scalingrectifiedflowtransformers,
      title={Scaling Rectified Flow Transformers for High-Resolution Image Synthesis}, 
      author={Patrick Esser and Sumith Kulal and Andreas Blattmann and Rahim Entezari and Jonas Müller and Harry Saini and Yam Levi and Dominik Lorenz and Axel Sauer and Frederic Boesel and Dustin Podell and Tim Dockhorn and Zion English and Kyle Lacey and Alex Goodwin and Yannik Marek and Robin Rombach},
      year={2024},
      eprint={2403.03206},
      archivePrefix={arXiv},
      primaryClass={cs.CV},
      url={https://arxiv.org/abs/2403.03206}, 
}

\end{document}